\documentclass[sn-mathphys-num]{sn-jnl}

\usepackage{graphicx} 
\usepackage{amsmath,amssymb} 
\usepackage{tikz}
\usepackage{algorithm}
\usepackage{algpseudocode}
\usepackage{booktabs}
\usepackage{rotating} 

\theoremstyle{thmstyleone}
\newtheorem{theorem}{Theorem}

\begin{document}

\title{RIFT: Relative Isolation From Trees For Anomaly Detection}

\author*[1]{M\'ark D\'aniel Szalai} 
\author*[1]{G\'abor Horv\'ath} 

\affil*[1]{%
  \orgdiv{Department of Networked Systems and Services, Faculty of Electrical Engineering and Informatics},
  \orgname{Budapest University of Technology and Economics}, 
  \orgaddress{\street{M\H{u}egyetem rkp. 3}, \city{Budapest}, \postcode{H-1111}, \country{Hungary}}
}

\email{szalaimd@edu.bme.hu} 
\email{horvath.gabor@vik.bme.hu} 

\abstract{Isolation Forest (IF) is a widely used baseline for unsupervised anomaly detection. Recent studies provide a closed-form expression for the infinite-forest limit for one-dimensional data. Inspired by the geometric interpretation of this formula, we introduce RIFT (Relative Isolation From Trees), a deterministic anomaly detection method that generates the minimum spanning tree and scores each point by the sum of the apparent sizes of tree edges as viewed from that point. For one-dimensional data, the RIFT score recovers the closed-form IF limit exactly. In higher dimensions, it provides a parameter-free generalization that is deterministic, robust to varying density and clustered anomalies and avoids the axis-parallel artifacts of IF. We further propose an ensemble variant for large datasets. Experiments on synthetic data and the ADBench benchmark demonstrate that the accuracy is comparable to IF, while the ensemble variant exhibits significantly lower variance across random seeds.}


\maketitle

\section{Introduction}

Today, vast amounts of data are generated across various domains, including IT systems, industry, finance, and healthcare \cite{chandola2009anomaly, aggarwal2017outlier}. Machine learning tools enable the automated detection of anomalies, or outliers, within these datasets. An anomaly is generally defined as a data point that deviates significantly from the majority, suggesting it was generated by a different underlying process \cite{hawkins1980identification}. Often, anomalies are identified simply to be removed, as they may result from measurement or data collection errors that could distort the analysis. In many other cases, however, the anomalies themselves hold the most valuable information: they might indicate diseases in medical diagnostics \cite{hauskrecht2013outlier}, cyberattacks in IT security \cite{sommer2010outside}, or subtle equipment malfunctions in industrial systems \cite{fujimaki2005approach} that are difficult for human operators to detect.

A vast number of anomaly detection methods has been proposed in the literature. These include techniques designed for time-series data \cite{gupta2013outlier}, graph-based methods \cite{akoglu2015graph}, and algorithms tailored for textual data (such as system logs) \cite{he2016experience}. Another major category focuses on detecting "point anomalies" within multidimensional tabular data \cite{goldstein2016comparative}, assuming no underlying sequential or temporal structure among the independent samples. In this paper, we introduce a novel algorithm that falls into this latter category.

While supervised machine learning can be used for anomaly detection, it faces two significant challenges. First, it requires large amounts of labeled training data \cite{emmott2013systematic}. Second, supervised models tend to overfit to known anomaly patterns rather than learning the boundaries of normal behavior, making them blind to novel, previously unseen types of outliers \cite{pimentel2014review}. The method proposed in this paper is strictly unsupervised. It does not require any pre-labeled training data, which is often impractical or impossible to obtain in real-world scenarios, and separates anomalies from normal instances based solely on the intrinsic properties of the data itself.

One of the most renowned algorithms in this unsupervised domain is the Isolation Forest (IF) \cite{liu2008iforest}. Based on a simple yet elegant concept, IF is robust, fast, and provides accurate results both on small and large datasets. The algorithm introduced in our paper shares conceptual similarities with the Isolation Forest but operates on entirely different mechanics. For one-dimensional data, our method yields the exact same results as an Isolation Forest containing an infinite number of trees, each grown on the full sample without a height limit. In higher dimensions, our approach utilizes an anomaly scoring metric analogous to IF, but computes it through a fundamentally different, deterministic process based on geometric distances rather than random recursive partitioning. Consequently, our algorithm is not a mere derivative of existing techniques, but a novel method that broadens the landscape of anomaly detection algorithms.

The rest of the paper is organized as follows. First, Section \ref{sec:related_work} reviews the related literature and anomaly detection methods. Following this, in Section \ref{sec:isolation_forest}, we revisit the Isolation Forest algorithm, detailing its closed-form asymptotic solution for 1D data and its geometric interpretation. The proposed algorithm is then introduced in Section \ref{sec:rift}, starting with the exact methodology, followed by a fast but highly accurate approximation, and concluding with a scalable ensemble-based variant designed for large datasets. In Section \ref{sec:numerical}, we evaluate the performance of our approach on synthetic data as well as standard benchmark datasets from ADBench, comparing it against several widely used anomaly detection techniques. Finally, Section \ref{sec:conclusions} provides the concluding remarks.

\section{Related work}\label{sec:related_work}
Unsupervised point-anomaly detection in tabular data has a long history and spans a wide range of modeling assumptions \cite{chandola2009anomaly,aggarwal2017outlier}. Early and widely used baselines include linear subspace methods based on Principal Component Analysis (PCA), which detect outliers via large reconstruction error or distance to a low-dimensional subspace \cite{shyu2003novel}.

A complementary line of work formulates anomaly detection as novelty detection via one-class classification. One-Class SVM (OCSVM) learns a decision boundary that encloses the bulk of the data in feature space and flags points falling outside as anomalies \cite{scholkopf1999support}.

Distance- and neighborhood-based methods score a point by its proximity to other samples. Classical $k$-nearest-neighbor (KNN) outlier scores use distances to the $k$-th neighbor (or related aggregates) and are attractive due to their simplicity, but can become expensive in high dimensions without approximate indexing \cite{ramaswamy2000efficient}.

Density-based approaches aim to separate points in sparse regions from those in dense regions. The Local Outlier Factor (LOF) compares a point's local density to the densities of its neighbors, thereby capturing local outliers even when global density varies \cite{breunig2000lof}. Cluster-based methods such as CBLOF first summarize the data into clusters and then assign higher outlierness to points that belong to small clusters or are far from large clusters \cite{he2003discovering}.

Histogram- and distribution-based detectors trade modeling flexibility for scalability. HBOS assumes feature independence and computes outlierness from per-feature histograms, resulting in very fast scoring on large datasets \cite{goldstein2012histogram}. More recent methods such as COPOD and ECOD avoid explicit density estimation by using copulas or empirical cumulative distribution functions to derive robust, parameter-light outlier scores \cite{li2020copod,li2022ecod}.

Ensemble methods provide another effective route to scalability and robustness. Isolation Forest (IF) isolates anomalies via random recursive partitioning and remains a strong general-purpose baseline \cite{liu2008iforest}. LODA takes an alternative lightweight ensemble approach: it builds multiple one-dimensional random projections and uses simple histograms in the projected spaces to score outliers efficiently \cite{pevny2016loda}.

Overall, these methods highlight a spectrum of trade-offs between computational cost, robustness to dimensionality and density variation, and the extent to which the detector exploits global versus local geometric structure. Our proposed method is most closely related to geometric, distance-driven detectors, but it derives its isolation score from the structure of a (possibly approximated) minimum spanning tree rather than from random partitions or explicit density models.

\section{Isolation Forest}\label{sec:isolation_forest}

\subsection{The basic algorithm}

Isolation Forest (IF) is an ensemble method that detects anomalies by how easily they can be \emph{isolated} via random recursive partitioning \cite{liu2008iforest}. Let $X=\{x_1,\ldots,x_n\}\subset\mathbb{R}^d$. An IF model builds $t$ isolation trees, each grown from a random subsample $X'\subset X$ of size $\psi$.

An isolation tree is constructed recursively. At a node containing points $S\subseteq X'$, sample a feature index $q\in\{1,\ldots,d\}$ uniformly at random and then sample a split value $p$ uniformly from $[\min_{x\in S} x^{(q)},\max_{x\in S} x^{(q)}]$. Split $S$ into left/right children by the threshold $p$ on feature $q$. Recursion stops when $|S|\le 1$ (or all points in $S$ are identical) or when a maximum depth is reached.

For a point $x$, let $h_{T_k}(x)$ be its path length (number of edges traversed from the root to the external node) in tree $T_k$. The forest-average path length for $x$ is
\begin{equation}
  \mathbb{E}[h(x)] \,=\, \frac{1}{t}\sum_{k=1}^{t} h_{T_k}(x).
\end{equation}
To normalize for sample size, IF uses the average path length of an unsuccessful search in a binary search tree,
\begin{equation}
  c(\psi) \,=\, 2H_{\psi-1} - \frac{2(\psi-1)}{\psi},
\end{equation}
where $H_m=\sum_{i=1}^{m} \frac{1}{i}$ is the $m$-th harmonic number. The standard anomaly score is then
\begin{align}\label{eq:sx}
  s(x,\psi) \,=\, 2^{-\mathbb{E}[h(x)]/c(\psi)},
\end{align}
so that smaller $\mathbb{E}[h(x)]$ (easier isolation) yields larger anomaly scores.

Many variants of the IF algorithm have been proposed. Some change the way samples are randomly split (e.g., Extended Isolation Forest \cite{hariri2019extended}), while others adapt it to the online setting (e.g., BWOAIF \cite{hannakbilateral} and IForestASD \cite{ding2013anomaly}). In this paper, we consider only the original algorithm.

\subsection{Explicit solution for one-dimensional data}

Theorem~3.5 of \cite{zheng2025iforestbias} gives an \emph{exact} closed-form expression for the expected depth (i.e., the expected path length) assigned by IF in the one-dimensional setting, given an ordered sample $x_1 < \cdots < x_n$, for isolation trees grown on the full sample without a height limit. The expected depth of point $x_i$, denoted by $\bar{h}(x_i)$, obtained at $t \to \infty$, can be written explicitly in terms of the spacings $x_j - x_{j-1}$ between adjacent points:

\begin{align}
\label{eq:1dsolution}
    \bar{h}(x_i)
=
\sum_{j=2}^{i} \frac{x_j - x_{j-1}}{x_i - x_{j-1}}
+
\sum_{j=i+1}^{n} \frac{x_j - x_{j-1}}{x_j - x_i}.
\end{align}

This result makes it possible to obtain the 1D IF anomaly score from the sample geometry, rather than only via Monte Carlo over random trees. Appendix \ref{sec:proof1d} provides an alternative proof that reaches the same results in a few elementary algebraic steps. 

The authors of \cite{zheng2025iforestbias} also argue that \eqref{eq:1dsolution} does not extend trivially to $d > 1$. In higher dimensions, the expected depth depends on the full \emph{relative configuration} of the points in $\mathbb{R}^d$, not merely on their ordering along a line, so different point arrangements can lead to genuinely different functional forms. Moreover, since each split is made along a randomly chosen coordinate, both the set of admissible splitting coordinates and the resulting transition structure may change after every split. As a consequence, the size of an exact state-space description of the process grows rapidly with the dimension.

\section{Description of the proposed method}\label{sec:rift}

\subsection{Concept of the algorithm}

Although the exact solution of $\bar{h}(x)$ is not known for $d>1$, we can develop new anomaly detection methods inspired by the geometric interpretation of the one-dimensional result. 

To understand the mechanics of the closed-form 1D solution \eqref{eq:1dsolution}, it is helpful to interpret the summation geometrically, as illustrated by the example in Figure~\ref{fig:ex1d}. Using the notation introduced in the example, with $x_4$ as the target point, the anomaly score given by~\eqref{eq:1dsolution} can be rewritten as

\begin{equation}
    \bar{h}(x_i) = \sum_{k, k\neq i} \frac{d_k}{\ell_k}.
    \label{eq:1d_geometric_sum}
\end{equation}

Each term in Equation \ref{eq:1d_geometric_sum} describes a geometric ratio from the perspective of the target point ($x_4$ in Figure \ref{fig:ex1d}):
\begin{itemize}
    \item \textbf{The numerator $d_k$} represents the length of the empty space between consecutive data points (the 1D edges). This captures the \textit{local sparsity} of that specific region.
    \item \textbf{The denominator $\ell_k$} represents the ``line-of-sight'' distance from the target point to the evaluated edge (specifically, its further endpoint). This penalizes the impact of an edge based on its distance from the target.
\end{itemize}

Geometrically, this quotient can be interpreted as the 1D \textit{apparent size} (or 1D visual angle) of the segment $d_k$ as observed from $x_4$. The isolation of a point is given by the sum of these apparent sizes. 

\begin{figure}
    \centering
\begin{tikzpicture}[
    point/.style={circle, fill=black, inner sep=2pt},
    target/.style={circle, fill=red, inner sep=2.5pt},
    edge/.style={thick, shorten >=2pt, shorten <=2pt},
    dist/.style={dashed, thick, blue!80}
]

\draw[gray, very thin] (-1,0) -- (10,0);

\node[point, label=above:$x_1$] (x1) at (0,0) {};
\node[point, label=above:$x_2$] (x2) at (1.5,0) {};
\node[point, label=above:$x_3$] (x3) at (3.0,0) {};
\node[target, label=above:$x_4$] (x4) at (5.0,0) {};
\node[point, label=above:$x_5$] (x5) at (6.5,0) {};
\node[point, label=above:$x_6$] (x6) at (8.0,0) {};
\node[point, label=above:$x_7$] (x7) at (9.5,0) {};

\draw[edge] (x1) -- node[above] {$d_1$} (x2);
\draw[edge] (x2) -- node[above] {$d_2$} (x3);
\draw[edge] (x3) -- node[above] {$d_3$} (x4);
\draw[edge] (x4) -- node[above] {$d_5$} (x5);
\draw[edge] (x5) -- node[above] {$d_6$} (x6);
\draw[edge] (x6) -- node[above] {$d_7$} (x7);

\draw[dist] (x3 |- 0,-0.8) -- node[below] {$\ell_3$} (x4 |- 0,-0.8);
\draw[dist] (x2 |- 0,-1.6) -- node[below] {$\ell_2$} (x4 |- 0,-1.6);
\draw[dist] (x1 |- 0,-2.4) -- node[below] {$\ell_1$} (x4 |- 0,-2.4);

\draw[dist] (x4 |- 0,-0.8) -- node[below] {$\ell_5$} (x5 |- 0,-0.8);
\draw[dist] (x4 |- 0,-1.6) -- node[below] {$\ell_6$} (x6 |- 0,-1.6);
\draw[dist] (x4 |- 0,-2.4) -- node[below] {$\ell_7$} (x7 |- 0,-2.4);

\foreach \x in {x1,x2,x3,x4,x5,x6,x7} {
    \draw[gray, dotted, very thin] (\x) -- (\x |- 0,-2.6);
}
\end{tikzpicture}
    \caption{Components of the explicit solution of the anomaly score for $x_4$}
    \label{fig:ex1d}
\end{figure}
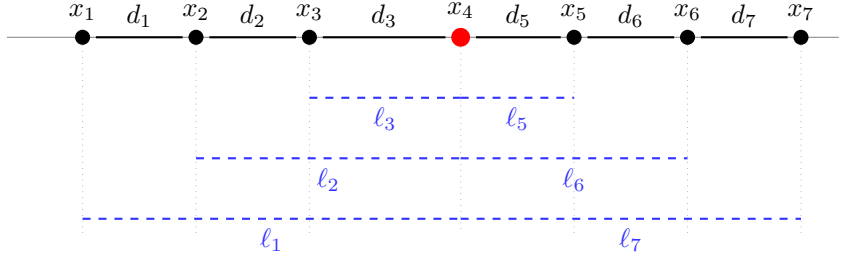

In the 1D scenario, the spatial ordering of data points is trivial and the terms of the summation ($d_k / \ell_k$) rely on a strictly defined sequence of edges and distances.  The topological structure is a simple linear chain. If we observe the environment from a target point, every other node in the sequence is associated with exactly one adjacent edge leading further away from the target. Consequently, there is a perfect one-to-one mapping between the evaluated nodes and the spatial gaps ($d_k$) they represent. 

In the sequel, we generalize this mathematical framework to higher-dimensional spaces ($d \ge 2$). In higher-dimensional space, however, edge selection is a critical, non-trivial question. To adapt the geometric ratio obtained in 1D, we must find a graph structure that preserves the aforementioned topological properties. We propose the \emph{Minimum Spanning Tree (MST)} based on Euclidean distances as the skeletal representation of the data. 

When an MST is constructed over the dataset, it creates a fully connected, cycle-free graph. If we temporarily view the target point $x_t$ as the root of this tree, the MST branches outwards to connect all other points. In any directed tree, every node (except the root) has exactly one incoming edge. Therefore, for any given node $x_k$, we can uniquely assign the length of the MST edge connecting it to its parent (in the direction of the target) as the numerator $d_k$. The denominator $\ell_k$ naturally generalizes to the direct Euclidean distance between the target point $x_t$ and the node $x_k$ (see Figure \ref{fig:ex2d}).

A normal point (deeply embedded within a cluster) is surrounded by numerous local edges at short distances ($\ell_k$). The sum of these ``apparent sizes'' ($d_k/\ell_k$) accumulates to a relatively large $\bar{h}(x_i)$ value, analogous to a deep terminal node in a standard Isolation Tree. An anomalous point lies far from most of the data. From such a point, dense clusters are distant (large $\ell_k$), while their internal edges ($d_k$) remain relatively small, shrinking their quotient to nearly zero. The primary contribution to the sum comes solely from the long bridging edge that connects the anomaly to the rest of the graph, making the total $\bar{h}(x_i)$ small. Defining the final anomaly score as the negative of $\bar{h}(x_i)$, the algorithm produces low scores for normal samples and high scores for anomalies.

This MST-based formulation provides an important theoretical consequence: if we apply this $d$-dimensional MST approach to one-dimensional data, it perfectly recovers the exact 1D closed-form solution given by eq. \eqref{eq:1d_geometric_sum}. In 1D, the MST naturally degenerates into the sequential linear chain connecting adjacent points, meaning that the MST edges $d_k$ and the distances $\ell_k$ become mathematically identical to those of our original 1D derivation.

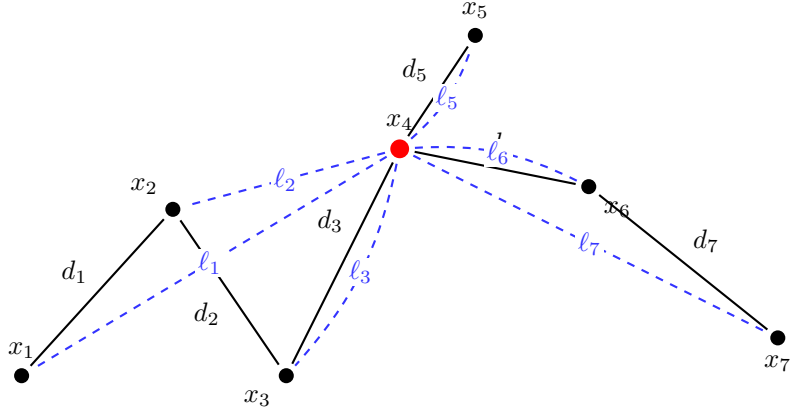
\begin{figure}
    \centering
\begin{tikzpicture}[
    point/.style={circle, fill=black, inner sep=2pt},
    target/.style={circle, fill=red, inner sep=2.5pt},
    edge/.style={thick, shorten >=2pt, shorten <=2pt},
    dist/.style={dashed, thick, blue!80, shorten >=2pt, shorten <=2pt}
]

\node[point, label=above:$x_1$] (x1) at (0, 0) {};
\node[point, label=above left:$x_2$] (x2) at (2, 2.2) {};
\node[point, label=below left:$x_3$] (x3) at (3.5, 0) {};
\node[target, label=above:$x_4$] (x4) at (5, 3) {};
\node[point, label=above:$x_5$] (x5) at (6, 4.5) {};
\node[point, label=below right:$x_6$] (x6) at (7.5, 2.5) {};
\node[point, label=below:$x_7$] (x7) at (10, 0.5) {};

\draw[edge] (x1) -- node[above left] {$d_1$} (x2);
\draw[edge] (x2) -- node[below left] {$d_2$} (x3);
\draw[edge] (x3) -- node[above left, pos=0.6] {$d_3$} (x4);
\draw[edge] (x4) -- node[above left] {$d_5$} (x5);
\draw[edge] (x4) -- node[above right, pos=0.4] {$d_6$} (x6);
\draw[edge] (x6) -- node[above right] {$d_7$} (x7);

\draw[dist] (x4) -- node[fill=white, inner sep=1pt] {$\ell_1$} (x1);
\draw[dist] (x4) -- node[fill=white, inner sep=1pt] {$\ell_2$} (x2);
\draw[dist] (x4) -- node[fill=white, inner sep=1pt] {$\ell_7$} (x7);

\draw[dist, bend left=15] (x4) to node[fill=white, inner sep=1pt] {$\ell_3$} (x3);
\draw[dist, bend right=15] (x4) to node[fill=white, inner sep=1pt] {$\ell_5$} (x5);
\draw[dist, bend left=15] (x4) to node[fill=white, inner sep=1pt] {$\ell_6$} (x6);

\end{tikzpicture}
    
    \caption{Components of the proposed algorithm in 2D}
    \label{fig:ex2d}
\end{figure}

Note that relying on the MST is a design choice. Since higher-dimensional spaces lack absolute ordering, the selection of the topological backbone is an open question. Other edge-selection strategies could also be employed to define $d_k$. Nevertheless, the MST provides a parameter-free, globally connected, and theoretically sound generalization.

The proposed geometric metric is inherently density-invariant. Since the topological relationship is evaluated as a quotient ($d_k/\ell_k$), scaling the density of a cluster proportionally scales both the local empty spaces ($d_k$) and the perspective distances ($\ell_k$), leaving the resulting sum effectively unchanged. This self-normalization allows the algorithm to recognize that a point within a sparse cloud is topologically just as normal within its own local context as a point in a tightly packed core.

Following the geometric interpretation described above, we name the proposed method RIFT (Relative Isolation From Trees).

\subsection{The exact solution using MST}\label{sec:exact}

The baseline algorithm is shown in Algorithm \ref{alg:global_mst_baseline}. To calculate the anomaly scores for all samples, it is enough to construct the MST only once. The main loop iterates over all sample points and selects them as target. The body of the loop is basically a breadth-first traversal of the MST starting from the target point. One traversal is enough to get the anomaly score for the target point. Note that $\bar{h}(x_i)$ is not the anomaly score itself, it is similar to the expected tree height of isolation forests, for which a lower value means higher anomaly score. It can be scaled to the $[0,1]$ range similar to eq. \eqref{eq:sx}, but it can also be used directly for the ROC-AUC tests.

\begin{algorithm}[h]
\caption{RIFT-exact: Global MST Isolation Score (Baseline)}
\label{alg:global_mst_baseline}
\begin{algorithmic}[1]
\Require Dataset $X = \{x_1, x_2, \dots, x_n\} \subset \mathbb{R}^D$
\Ensure Anomaly scores $S = \{s_1, s_2, \dots, s_n\}$
\vspace{1mm}

\State \textbf{Initialize} a complete graph $G = (X, E)$ where $weight(x_i, x_j) = \|x_i - x_j\|_2$
\State Construct the Minimum Spanning Tree (MST) from $G$ 
\State Initialize an empty array $S$ of size $n$

\For{each target point $x_t \in X$}
    \State $sum_t \gets 0$
    \State Initialize a queue $Q \gets \{x_t\}$
    \State Initialize a set of visited nodes $V \gets \{x_t\}$
    
    \While{$Q$ is not empty}
        \State Pop current node $u$ from $Q$
        
        \For{each neighbor $v$ of $u$ in the MST}
            \If{$v \notin V$}
                \State Add $v$ to $V$
                \State Push $v$ to $Q$
                
                \State $d_v \gets \|x_u - x_v\|_2$ 
                \State $\ell_v \gets \|x_t - x_v\|_2$ 
                
                    \State $sum_t \gets sum_t + \frac{d_v}{\ell_v}$
            \EndIf
        \EndFor
    \EndWhile
    
    \State $s_t \gets -sum_t$ \Comment{Negative sum assigns higher scores to anomalies}
\EndFor

\State \Return $S$
\end{algorithmic}
\end{algorithm}

The main drawback of the algorithm is the computational complexity and memory requirement of the MST construction, due to which it can be applied only on datasets of moderate size.

\subsection{Using MST approximations}\label{sec:approx}

While computing $\bar{h}(x)$ with the exact MST (Section \ref{sec:exact}) provides a strong baseline, constructing and traversing a single global MST becomes impractical on large datasets ($N \gg 10^4$).

To reduce this cost, we approximate the MST using a sparse $k$-nearest-neighbor (kNN) graph: for each sample we keep only edges to its $k$ closest neighbors and compute an MST on this reduced edge set.

We start from the kNN graph and build an MST on it. If the kNN graph is not connected, we increase $k$ and repeat until the graph (and thus the resulting MST approximation) becomes connected.

Let $n$ be the number of vertices (samples) and let $e$ be the number of edges in the graph. Classical MST algorithms such as Kruskal's run in $O(e\log e)$ time. For a fully connected graph we have $e = \frac{n(n-1)}{2} = O(n^2)$, which makes exact MST construction expensive. In contrast, a kNN graph has roughly $e \approx kn$ edges, which is far cheaper when $k \ll n$.

In some datasets we encountered many duplicate samples (thousands of points at exactly the same location). This can force $k$ to grow very large before the kNN graph becomes connected, which hurts runtime. A simple workaround is to build the approximate MST on the set of unique samples and then map duplicates back to their representative node. This solution keeps the graph sparse while still representing every sample.

\subsection{Ensemble-based implementation}\label{sec:ensemble}

 As another option to avoid the calculation and the traversal of huge MSTs and to enhance the robustness of the detector, we propose an ensemble-based approach inspired by the sub-sampling and feature bagging (random subspace) strategies. Instead of constructing a single global tree, we build a forest of $M$ independent Minimum Spanning Trees. Each tree $T_j$ is constructed from a small, randomly selected subset of the data instances $\psi \subset X$ (typically $|\psi| = 256$), and evaluated over a randomly selected subset of features $F_j \subseteq \{1, \dots, D\}$.

The isolation score of a target point $x_i$ (which may or may not have been part of the training subset) is evaluated as follows. For a given tree $T_j$, the geometric distances are computed exclusively within its respective feature subspace $F_j$. We first identify the connection point (anchor) $x^* \in T_j$ that is closest to the target point in this subspace:
\begin{align}
x^* = \arg \min_{x_k \in T_j} \|x_i - x_k\|_{F_j}.
\end{align}
The target point $x_i$ is attached to the tree as a new leaf node connected to $x^*$ via an edge of length $d_* = \|x_i - x^*\|_{F_j}$. The isolation value for the point $x_i$ relative to the tree $T_j$ is then computed as the sum of the apparent sizes of all edges in the augmented tree:
\begin{align}
    \bar{h}(x_i, T_j) = \frac{d_*}{\ell_*} + \sum_{k \in T_j, x_k \neq x^*} \frac{d_k}{\ell_k}.\label{eq:ensemble_score}
\end{align}
In eq. \eqref{eq:ensemble_score}, for all nodes $k$ already present in $T_j$, the numerator $d_k$ is the pre-calculated MST edge weight within the subspace $F_j$, while the denominator $\ell_k$ is the Euclidean distance between $x_i$ and $x_k$ in the same subspace. The final isolation score is obtained by averaging the results over the entire forest: 
\begin{align}
\bar{h}_{forest}(x_i) = \frac{1}{M} \sum_{j=1}^M \bar{h}(x_i, T_j).    
\end{align}

This ensemble approach significantly reduces the computational overhead, as the complexity of constructing and traversing $M$ small trees is significantly lower than that of a single global MST. The forest structure combined with feature bagging introduces further statistical advantages. Sub-sampling reduces the \textit{masking effect}, where a tight cluster of anomalies might otherwise appear as a dense normal region according to the global detector. By isolating members of such micro-clusters in different sub-samples, they get revealed as outliers. The ensemble approach also prevents \textit{swamping}, ensuring that redundant normal instances in high-density regions do not inflate the isolation scores of distant anomalies. Furthermore, feature bagging mitigates the curse of dimensionality, preventing irrelevant or noisy features from dominating the distance calculations across the entire forest.

\begin{algorithm}[h]
\caption{RIFT-ensemble: Sub-sampled MST-Forest with Feature Bagging}
\label{alg:ensemble_mst}
\begin{algorithmic}[1]
\Require Dataset $X = \{x_1, \dots, x_n\} \subset \mathbb{R}^D$, Number of trees $M$, Sub-sample size $\psi$, subset size $q \le D$
\Ensure Ensemble anomaly scores $S = \{s_1, \dots, s_n\}$
\vspace{1mm}

\State Initialize an array $S$ of size $n$ with zeros

\For{$m = 1$ \textbf{to} $M$}
    \State Draw a random subset of instances $X_m \subset X$ of size $\psi$
    \State Draw a random subset of features $F_m \subseteq \{1, \dots, D\}$ of size $q$
    \State Construct a complete graph $G_m$ on $X_m$ where $weight(u, v) = \|u - v\|_{F_m}$
    \State Construct the Minimum Spanning Tree $T_m$ from $G_m$

    \For{each target point $x_t \in X$}
        \State \textit{// Phase 1: Topological Grafting (Anchor Selection in Subspace)}
        \State Find anchor node $x^* \gets \arg\min_{v \in X_m} \|x_t - v\|_{F_m}$
        \State $d_{hook} \gets \|x_t - x^*\|_{F_m}$ 
        \State $\ell_{hook} \gets \|x_t - x^*\|_{F_m}$ 
        \State $sum_{t, m} \gets \frac{d_{hook}}{\ell_{hook}}$

        \State \textit{// Phase 2: Tree Traversal (Referencing Algorithm \ref{alg:global_mst_baseline})}
        \State \textbf{Traverse} $T_m$ via BFS starting from $x^*$ identically to Algorithm \ref{alg:global_mst_baseline}:
        \Statex \qquad For each newly visited node $v$, accumulate $\frac{d_v}{\ell_v}$ into $sum_{t, m}$, where:
        \Statex \qquad \quad $d_v$ is the incoming edge weight in $T_m$
        \Statex \qquad \quad $\ell_v = \|x_t - v\|_{F_m}$
        
        \State $S[t] \gets S[t] + sum_{t, m}$
    \EndFor
\EndFor

\For{$t = 1$ \textbf{to} $n$}
    \State $s_t \gets -\frac{S[t]}{M}$ \Comment{Average over the forest and negate}
\EndFor

\State \Return $S$
\end{algorithmic}
\end{algorithm}

It is important to note that while the introduction of data sub-sampling and feature bagging makes the algorithm non-deterministic, the MST-Forest exhibits a higher level of structural consistency than the traditional Isolation Forest. The Isolation Forest relies on pervasive randomness: stochastic sub-sampling followed by completely random axis selection and split-point partitioning at every single node of every tree. Our method limits randomness strictly to the initial instance and feature sampling phase of each tree. Once a data subset and a feature subspace are selected, the resulting MST provides a deterministic representation of the local topology within that specific view. Consequently, the MST-Forest maintains the robustness of ensemble methods while exhibiting lower variance in its anomaly scores.

\section{Numerical study}\label{sec:numerical}

The proposed method was implemented in Python, utilizing the Numba library for computational optimization. To ensure accessibility and ease of integration, the implementation provides an API that is fully consistent with standard \textit{scikit-learn} estimators.

Since the algorithm relies on Euclidean distances, proper feature scaling prior to execution is critical. However, because anomalies frequently distort data distributions in an asymmetric manner, traditional standardization or min-max scaling are not optimal. For the numerical experiments presented in this paper, we applied a robust scaling approach that can properly handle the influence of extreme outliers: we used the median and the MAD, applied on the unique data points, to scale each feature.

\subsection{Anomaly heatmaps with synthetic data}


The way anomalies are distributed strongly affects how well different detectors perform. In this section we show that the proposed method handles two situations that are difficult for many detectors: datasets with varying density, and anomalies that form tight clusters. It handles both by construction, without hyperparameters, and deterministically.


We use the three synthetic datasets shown in Figure~\ref{fig:synthds}. The first contains two perpendicular Gaussian blobs, with the anomalies grouped around two points. The second illustrates varying density: the normal data form one sparse and one dense cloud, while the anomalies are arranged in a ring around the dense cloud. In the third, the anomalies form a single tight cluster.

\begin{figure}
    \centering
    \includegraphics[width=\linewidth]{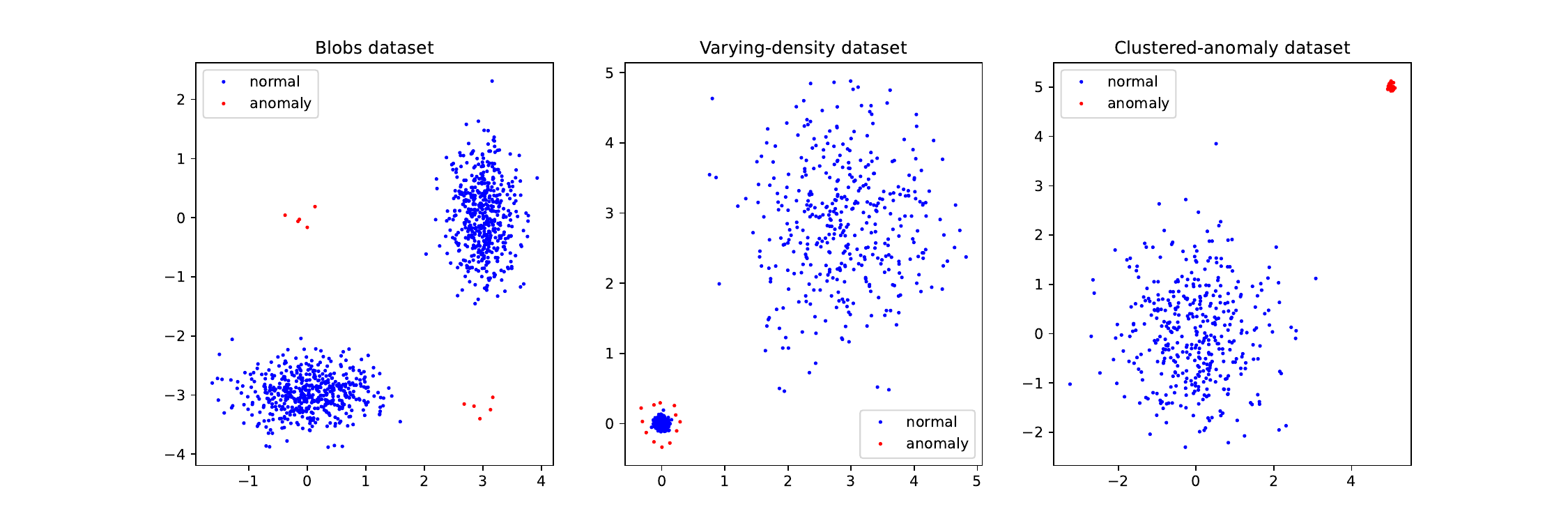}
    \caption{Synthetic datasets used in the comparison}
    \label{fig:synthds}
\end{figure}


The methods compared are Isolation Forest, KNN (which we will later see is surprisingly effective on real data), LOF (which handles varying-density data particularly well), and two further methods, HBOS and ECOD.


The ROC AUC results are summarised in Table~\ref{tab:synthscore}. LOF and KNN perform very well on the Blob and Varying-density datasets but poorly on the Clustered-anomalies dataset, while HBOS and ECOD show the opposite pattern. Isolation Forest and the proposed method achieve high ROC AUC in both regimes — the proposed method while also producing deterministic results.

\begin{table}[h]
    \centering
    \begin{tabular}{l|ccc}
        \toprule
        Method & Blob & Varying density & Clustered anomalies \\
        \midrule
        Isolation Forest & 0.995 & 0.961 & 0.956 \\
        LOF & 0.999 & 1.000 & 0.365 \\
        KNN & 0.995 & 0.846 & 0.002 \\
        HBOS & 0.324 & 0.432 & 0.947 \\
        ECOD & 0.526 & 0.656 & 0.983\\
        RIFT (proposed) & 0.999 & 0.982 & 0.993\\
        \bottomrule
    \end{tabular}
    \caption{ROC AUC scores on the synthetic datasets}
    \label{tab:synthscore}
\end{table}


The heatmaps (Figures~\ref{fig:blobsheatmap}, \ref{fig:varyingheatmap}, \ref{fig:clusteredheatmap}) are informative on how the methods behave. Isolation Forest produces axis-parallel artifacts, LOF radial ones, and HBOS a blocky, axis-parallel surface. For KNN, anomalies close to a cluster are assigned normal scores too quickly. RIFT, the proposed method, yields a smooth anomaly surface, free of such artifacts.

\begin{figure}
    \centering
    \includegraphics[width=\linewidth]{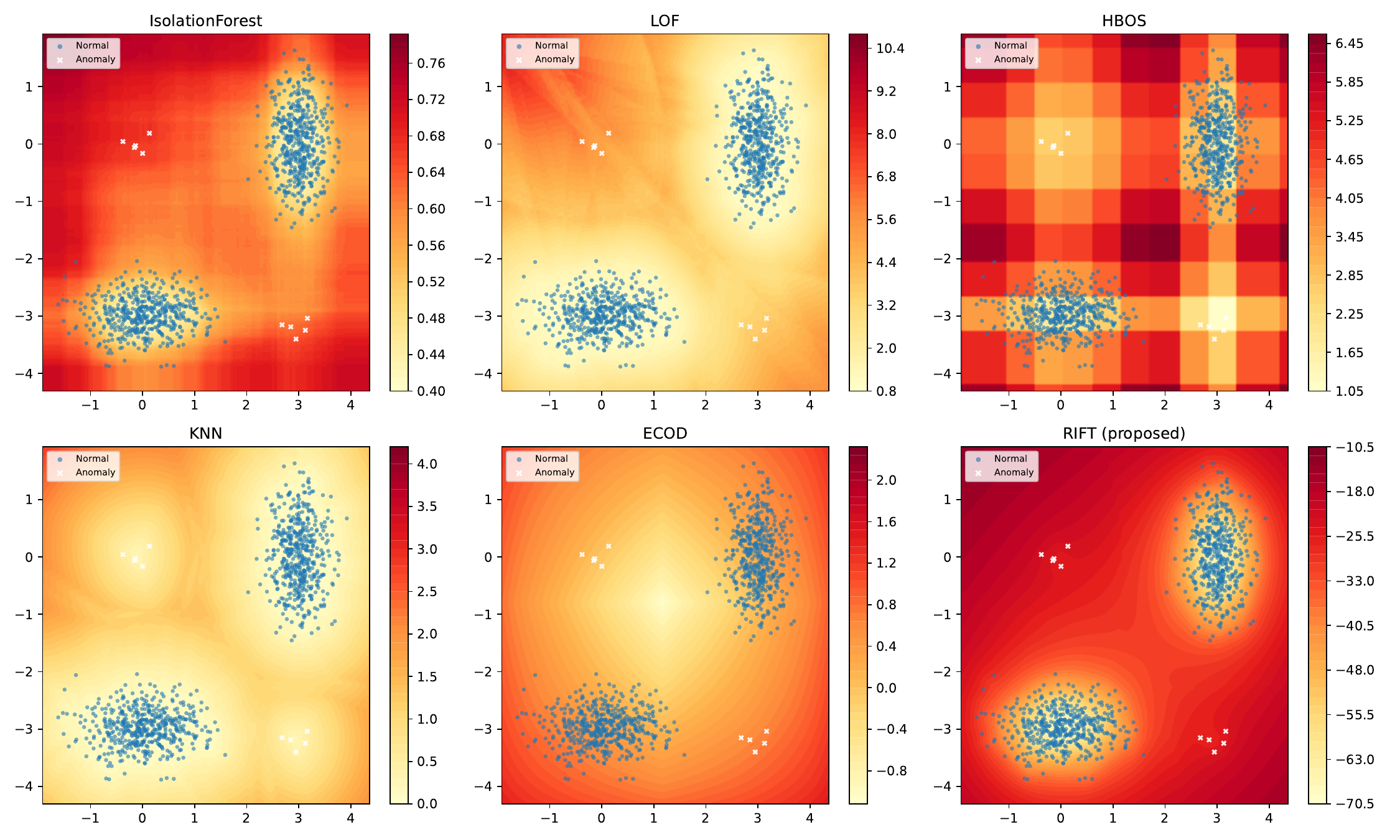}
    \caption{Anomaly scores on the ``blobs'' dataset}
    \label{fig:blobsheatmap}
\end{figure}
\begin{figure}
    \centering
    \includegraphics[width=\linewidth]{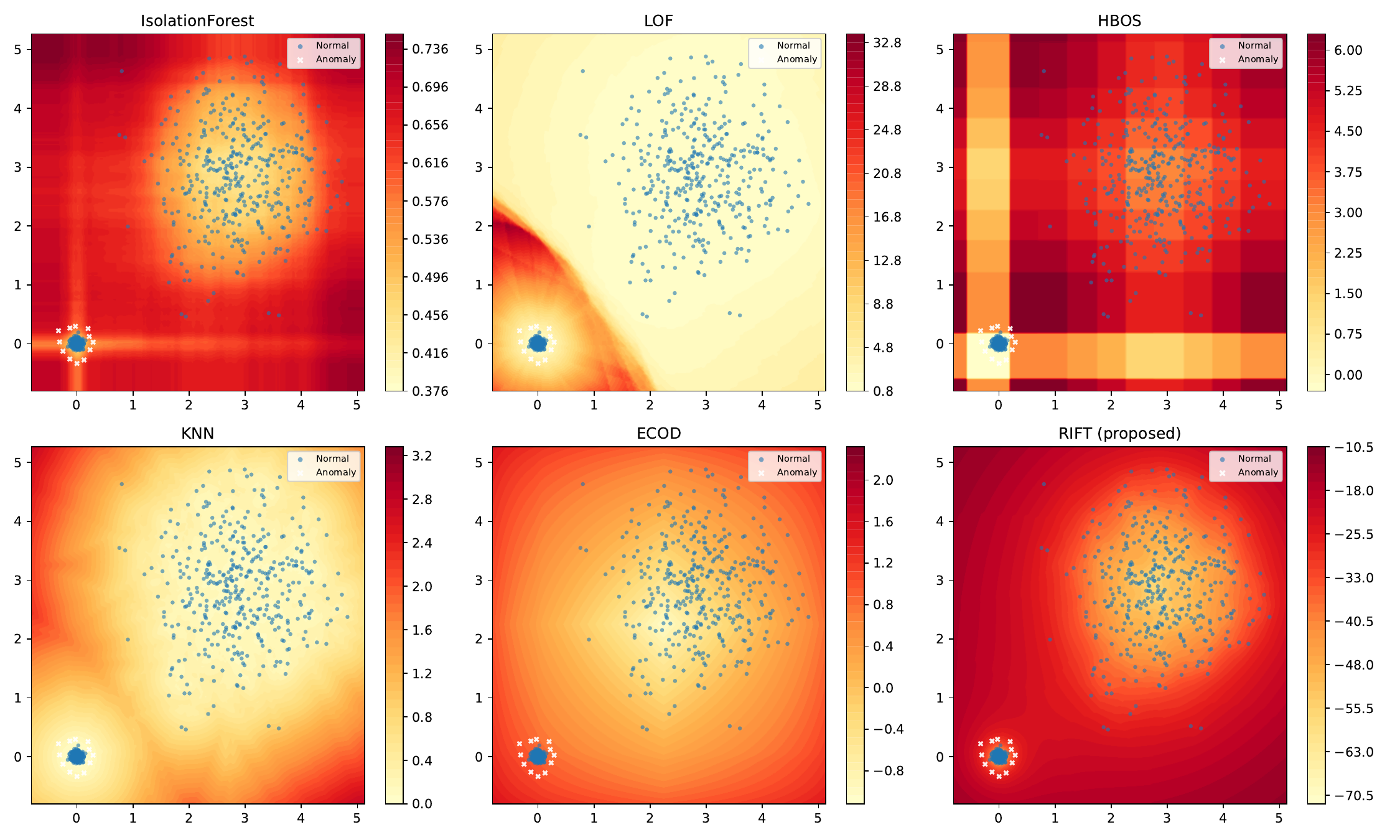}
    \caption{Anomaly scores on the ``varying density'' dataset}
    \label{fig:varyingheatmap}
\end{figure}
\begin{figure}
    \centering
    \includegraphics[width=\linewidth]{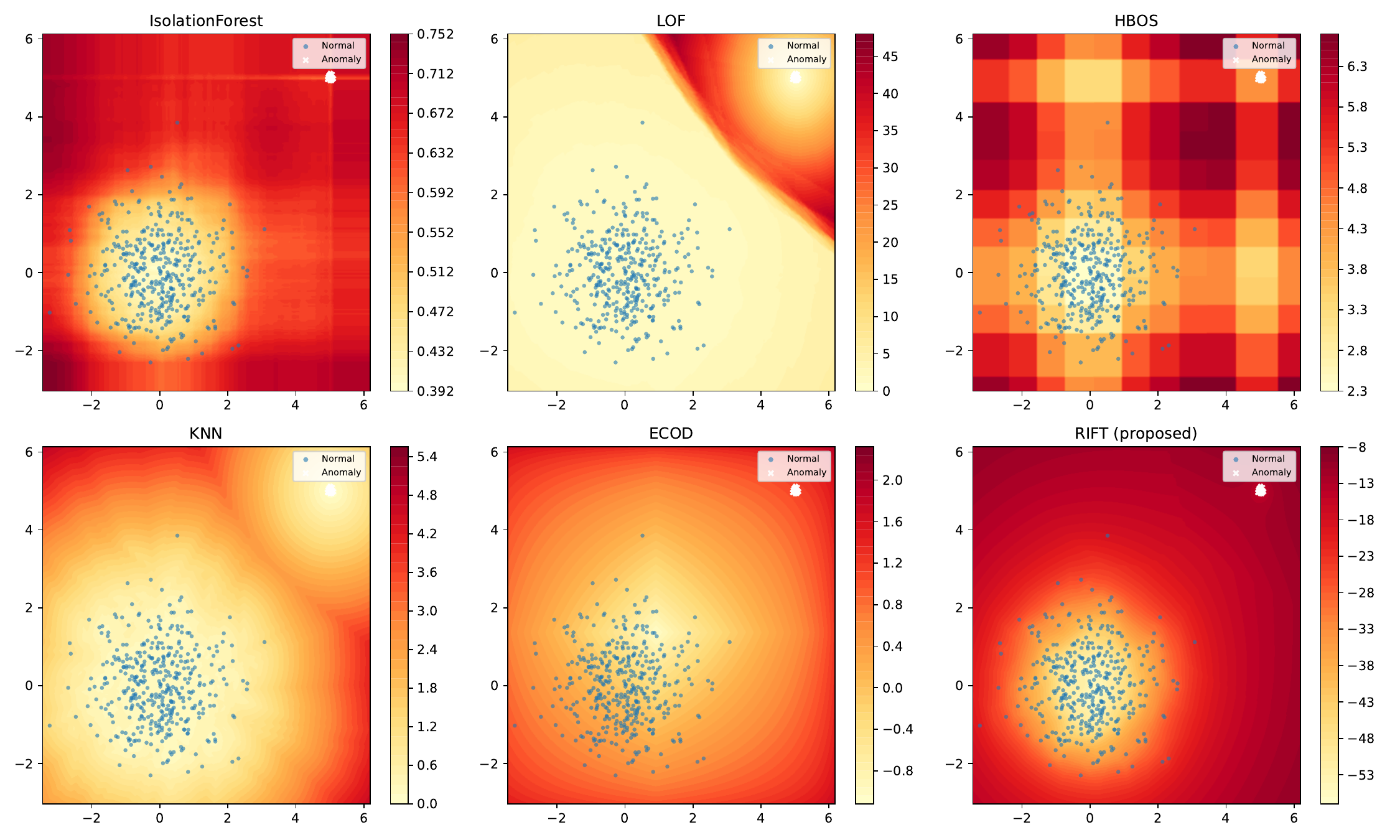}
    \caption{Anomaly scores on the ``clustered anomalies'' dataset}
    \label{fig:clusteredheatmap}
\end{figure}


Thanks to the use of a reachability-based distance, LOF is known to handle varying-density data well, and to some extent clustered anomalies. As Figure~\ref{fig:synthauc} shows, LOF ($k=10$) performs well on the varying-density dataset, but on the clustered-anomaly dataset its performance collapses once the anomaly cluster grows larger than its parameter $k$. On that dataset ECOD performs well. The proposed method is strong in both cases, and achieves this without any hyperparameter.

\begin{figure}
    \centering
    \includegraphics[width=0.49\linewidth]{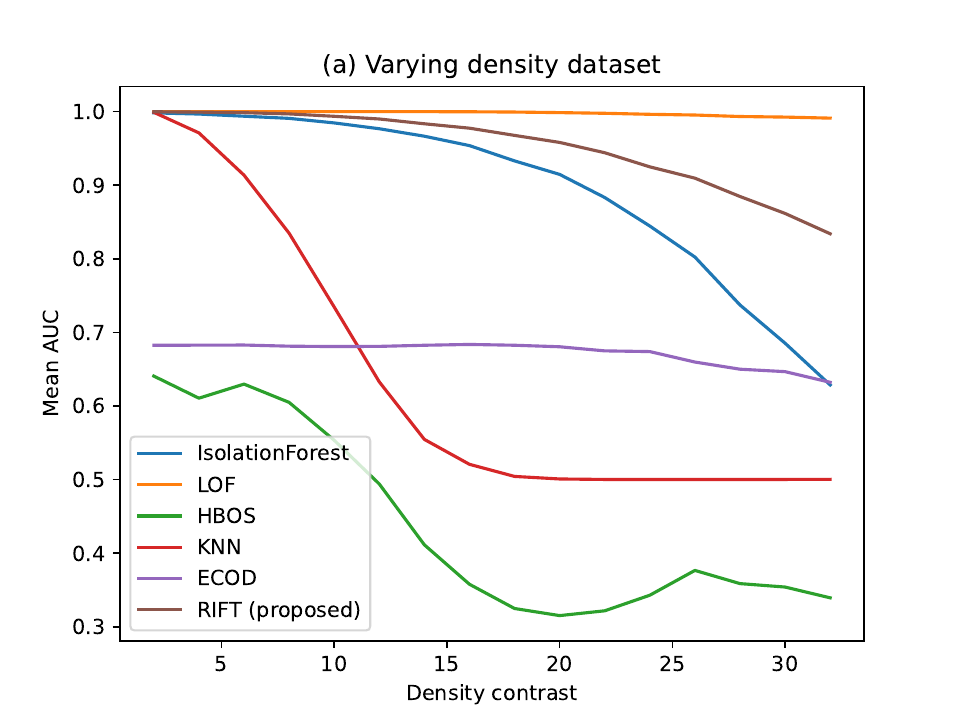}
    \includegraphics[width=0.49\linewidth]{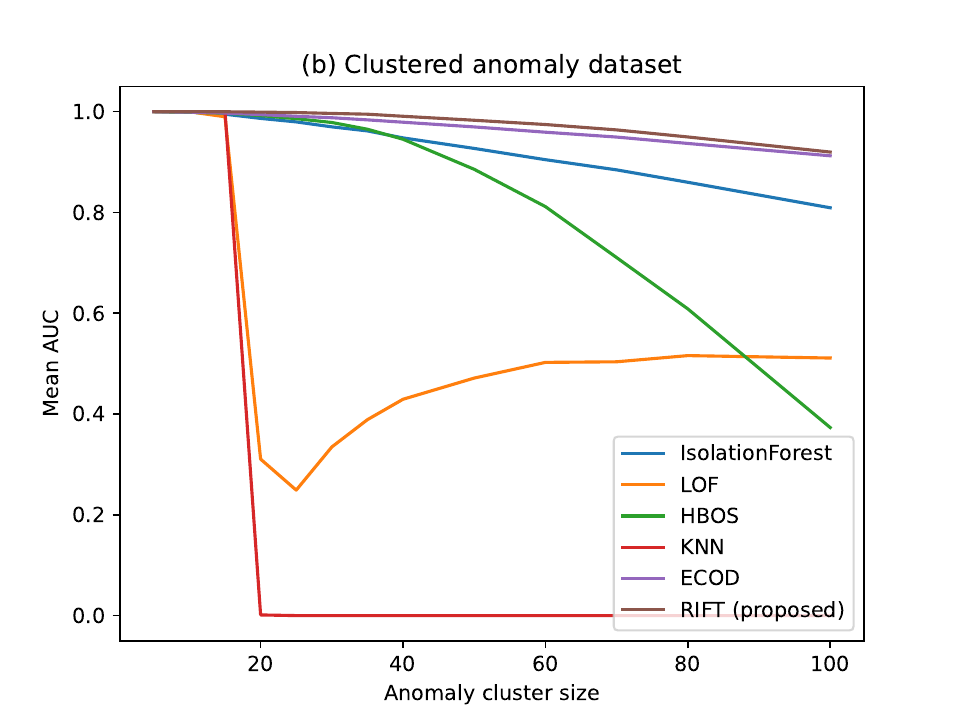}
    \caption{Performance of anomaly detection methods on two synthetic datasets}
    \label{fig:synthauc}
\end{figure}


No anomaly detector is perfect, each has its own weakness. In the following sections we use real data to also show which kinds of anomalies the proposed method handles less well.

\subsection{Evaluation of the ensemble based approach}

The ensemble-based approach is characterized by three hyper-parameters:
\begin{itemize}
\item the ensemble size (the total number of generated trees),
\item the number of data samples per tree (the degree of data sub-sampling),
\item the feature bagging ratio (the fraction of the total dimensions considered during the construction of a given tree).
\end{itemize}
In this section, we investigate the empirical impact of these parameters using several standard benchmark datasets.

First, feature bagging is disabled to investigate the impact of the remaining two hyper-parameters. Initially, both parameters are set to sufficiently large baseline values and their effects are observed by systematically decreasing them one at a time. The evaluation is conducted using the benchmark datasets detailed in Table \ref{tab:datasets}, which are diverse in terms of sample size, dimensionality, and application domain. Each configuration is executed 100 times using different random seeds.

\begin{table}[]
    \centering
    \begin{tabular}{l|c|c|c}
        \toprule
         \textbf{Dataset} & \textbf{Samples} & \textbf{Features} & \textbf{Category}  \\
         \midrule
         annthyroid & 7200 & 6 & Healthcare \\
         magic.gamma & 19020 & 10 & Physical \\
         mammography & 11183 & 6 &  Healthcare \\
         mnist & 7603 & 100 & Image \\
         satellite & 6435 & 36 & Astronautics \\
         \bottomrule
    \end{tabular}
    \caption{Datasets used in the comparison}
    \label{tab:datasets}
\end{table}

The top part of Table \ref{tab:ens} presents the average AUC results across 100 independent runs for various ensemble size configurations. The results achieved with an exceptionally small number of trees are remarkably close to those obtained with massive forests. However, increasing the size of the ensemble provides a clear benefit in terms of the variance of the results. As shown in the middle part of Table \ref{tab:ens}, the already low variance observed with a few trees approaches zero as the number of trees exceeds one hundred.

The results also indicate that, in most scenarios, the ensemble approach closely approximates the performance of the exact MST baseline while being superior in terms of computational efficiency. 

However, for the \textit{satellite} dataset the ensemble method performed significantly worse than the exact approach. The reason could be that by restricting the trees to small sub-samples, the ensemble fails to capture the complete global manifold, resulting in a loss of structural information.

However, in the case of the \textit{mammography} dataset, the ensemble method outperformed the exact baseline. Our investigations revealed that this dataset contains a relatively high number of coincident or near-coincident anomalous samples. These tightly packed micro-clusters trigger a \textit{masking} effect: the exact MST algorithm perceives the near-zero distances between these identical or highly similar anomalies as evidence of a dense, normal region. The sub-sampling of the ensemble resolves this issue. By constructing trees from small random subsets, these dense anomaly clusters are broken apart, the anomalies lose their masking and are forced to connect to the tree via longer edges, making the anomalies easier to identify.

\begin{table}
\centering
\begin{tabular}{l|c|ccccccc}
\toprule
 & exact & \multicolumn{7}{c}{Ensemble size =} \\
Dataset & MST & 4 & 8 & 16 & 32 & 64 & 128 & 256 \\
\midrule
 \multicolumn{9}{c}{Mean AUC scores across seeds}\\
\midrule
annthyroid & 0.887 & 0.888 & 0.888 & 0.888 & 0.888 & 0.888 & 0.888 & 0.888 \\
magic.gamma & 0.693 & 0.695 & 0.694 & 0.694 & 0.695 & 0.695 & 0.695 & 0.695 \\
mammography & 0.836 & 0.858 & 0.859 & 0.859 & 0.859 & 0.859 & 0.859 & 0.859 \\
mnist & 0.845 & 0.840 & 0.840 & 0.840 & 0.840 & 0.840 & 0.840 & 0.840 \\
satellite & 0.692 & 0.678 & 0.677 & 0.677 & 0.677 & 0.677 & 0.677 & 0.677 \\
\midrule
\multicolumn{9}{c}{Standard deviation of AUC scores across seeds}\\
\midrule
annthyroid & 0.000 & 0.004 & 0.002 & 0.002 & 0.001 & 0.001 & 0.001 & 0.001 \\
magic.gamma & 0.000 & 0.005 & 0.003 & 0.002 & 0.002 & 0.001 & 0.001 & 0.001 \\
mammography & 0.000 & 0.005 & 0.003 & 0.002 & 0.001 & 0.001 & 0.001 & 0.001 \\
mnist & 0.000 & 0.007 & 0.005 & 0.004 & 0.002 & 0.002 & 0.001 & 0.001 \\
satellite & 0.000 & 0.006 & 0.004 & 0.003 & 0.002 & 0.001 & 0.001 & 0.001 \\
\midrule
\multicolumn{9}{c}{Execution times [s]}\\
\midrule
annthyroid & 0.461 & 0.007 & 0.015 & 0.027 & 0.051 & 0.097 & 0.196 & 0.385 \\
magic.gamma & 3.380 & 0.016 & 0.029 & 0.058 & 0.116 & 0.227 & 0.452 & 0.901 \\
mammography & 0.993 & 0.009 & 0.017 & 0.033 & 0.065 & 0.129 & 0.256 & 0.509 \\
mnist & 1.084 & 0.022 & 0.043 & 0.083 & 0.163 & 0.320 & 0.633 & 1.257 \\
satellite & 0.460 & 0.012 & 0.023 & 0.045 & 0.088 & 0.173 & 0.345 & 0.684 \\
\bottomrule
\end{tabular}
    \caption{The effect of different ensemble sizes}
    \label{tab:ens}
\end{table}

Next, we investigate the influence of the sub-sample size (the number of data samples per tree) while keeping the ensemble size fixed at $100$, which aligns with the default parameter of the traditional Isolation Forest. As demonstrated in Table \ref{tab:smpl}, this hyperparameter exhibits a more pronounced impact on the overall performance than the total number of trees. For the majority of the datasets, increasing the sub-sample size systematically improves the AUC scores. Interestingly, for certain datasets such as \textit{mammography}, restricting the trees to a significantly smaller sample size yields better accuracy. As detailed earlier, this phenomenon follows from the masking effect: trees constructed from fewer samples can more effectively break apart relatively dense clusters of anomalies.

The standard deviations of the predictions are summarized in the middle part of Table \ref{tab:smpl}. The results indicate a clear trend: as the number of instances per tree increases, the variance of the ensemble decreases.

\begin{table}
\centering
\begin{tabular}{l|c|ccccccc}
\toprule
 & exact & \multicolumn{7}{c}{Number of samples per tree =} \\
Dataset & MST & 4 & 8 & 16 & 32 & 64 & 128 & 256 \\
\midrule
 \multicolumn{9}{c}{Mean AUC scores across seeds}\\
\midrule
annthyroid & 0.887 & 0.865 & 0.875 & 0.881 & 0.885 & 0.887 & 0.888 & 0.888 \\
magic.gamma & 0.693 & 0.688 & 0.690 & 0.693 & 0.695 & 0.695 & 0.695 & 0.695 \\
mammography & 0.836 & 0.914 & 0.904 & 0.892 & 0.881 & 0.872 & 0.865 & 0.859 \\
mnist & 0.845 & 0.820 & 0.827 & 0.832 & 0.835 & 0.838 & 0.839 & 0.840 \\
satellite & 0.692 & 0.611 & 0.626 & 0.641 & 0.654 & 0.665 & 0.672 & 0.677 \\
\midrule
\multicolumn{9}{c}{Standard deviation of AUC scores across seeds}\\
\midrule
annthyroid & 0.000 & 0.007 & 0.004 & 0.003 & 0.002 & 0.001 & 0.001 & 0.001 \\
magic.gamma & 0.000 & 0.007 & 0.005 & 0.004 & 0.003 & 0.002 & 0.001 & 0.001 \\
mammography & 0.000 & 0.003 & 0.003 & 0.003 & 0.002 & 0.002 & 0.001 & 0.001 \\
mnist & 0.000 & 0.013 & 0.008 & 0.005 & 0.004 & 0.003 & 0.002 & 0.001 \\
satellite & 0.000 & 0.009 & 0.006 & 0.004 & 0.003 & 0.003 & 0.002 & 0.001 \\
\midrule
\multicolumn{9}{c}{Execution times [s]}\\
\midrule
annthyroid & 0.461 & 0.021 & 0.022 & 0.025 & 0.030 & 0.043 & 0.072 & 0.149 \\
magic.gamma & 3.380 & 0.034 & 0.038 & 0.046 & 0.063 & 0.097 & 0.172 & 0.349 \\
mammography & 0.993 & 0.025 & 0.027 & 0.031 & 0.039 & 0.057 & 0.098 & 0.196 \\
mnist & 1.084 & 0.024 & 0.030 & 0.042 & 0.071 & 0.129 & 0.230 & 0.486 \\
satellite & 0.460 & 0.020 & 0.023 & 0.029 & 0.040 & 0.071 & 0.131 & 0.269 \\
\bottomrule
\end{tabular}
    \caption{The effect of the number of samples per tree}
    \label{tab:smpl}
\end{table}

Finally, we investigate the impact of the feature bagging ratio. As shown in the upper part of Table \ref{tab:feat}, this hyperparameter has a substantial effect on the overall performance. For datasets where every feature is informative and plays a crucial role in detecting anomalies (such as \textit{annthyroid}, \textit{magic.gamma} or \textit{satellite}), it is optimal to set this parameter to a high value. Conversely, some datasets contain numerous uninformative features that act as noise and degrade the algorithm's detection capabilities. Feature bagging enables the generation of trees that incorporate fewer of these noisy dimensions. This is particularly evident with the \textit{mammography} dataset, which achieved its highest score at a setting of $0.6$, and the \textit{mnist} dataset, which is better handled with a bagging ratio of $0.2$. Furthermore, the table demonstrates that among all the hyper-parameters, the feature bagging ratio has the most pronounced effect on increasing the variance of the results.

\begin{table}
\centering
\begin{tabular}{l|c|ccccccc}
\toprule
 & exact & \multicolumn{7}{c}{Feature bagging ratio =} \\
Dataset & MST & 0.1 & 0.2 & 0.3 & 0.4 & 0.6 & 0.8 & 1.0 \\
\midrule
 \multicolumn{9}{c}{Mean AUC scores across seeds}\\
\midrule
annthyroid & 0.887 & 0.840 & 0.840 & 0.840 & 0.858 & 0.877 & 0.887 & 0.888 \\
magic.gamma & 0.693 & 0.687 & 0.689 & 0.690 & 0.692 & 0.693 & 0.694 & 0.695 \\
mammography & 0.836 & 0.848 & 0.848 & 0.848 & 0.860 & 0.867 & 0.840 & 0.859 \\
mnist & 0.845 & 0.744 & 0.843 & 0.840 & 0.840 & 0.840 & 0.840 & 0.840 \\
satellite & 0.692 & 0.675 & 0.675 & 0.675 & 0.676 & 0.676 & 0.677 & 0.677 \\
\midrule
\multicolumn{9}{c}{Standard deviation of AUC scores across seeds}\\
\midrule
annthyroid & 0.000 & 0.025 & 0.025 & 0.025 & 0.022 & 0.016 & 0.010 & 0.001 \\
magic.gamma & 0.000 & 0.011 & 0.008 & 0.006 & 0.005 & 0.004 & 0.002 & 0.001 \\
mammography & 0.000 & 0.009 & 0.009 & 0.009 & 0.008 & 0.022 & 0.007 & 0.001 \\
mnist & 0.000 & 0.051 & 0.011 & 0.009 & 0.006 & 0.005 & 0.003 & 0.001 \\
satellite & 0.000 & 0.005 & 0.003 & 0.003 & 0.002 & 0.002 & 0.002 & 0.001 \\
\midrule
\multicolumn{9}{c}{Execution times [s]}\\
\midrule
annthyroid & 0.461 & 0.128 & 0.127 & 0.125 & 0.129 & 0.131 & 0.139 & 0.150 \\
magic.gamma & 3.380 & 0.264 & 0.237 & 0.246 & 0.266 & 0.307 & 0.292 & 0.349 \\
mammography & 0.993 & 0.159 & 0.159 & 0.158 & 0.162 & 0.168 & 0.181 & 0.195 \\
mnist & 1.084 & 0.184 & 0.219 & 0.286 & 0.321 & 0.419 & 0.475 & 0.482 \\
satellite & 0.460 & 0.126 & 0.158 & 0.164 & 0.182 & 0.210 & 0.231 & 0.271 \\
\bottomrule
\end{tabular}
    \caption{The effect of feature bagging ratio}
    \label{tab:feat}
\end{table}

\subsection{Evaluation on the ADBench datasets}

In this section, we evaluate the performance of our proposed method against several widely used anomaly detection algorithms. For this comparative analysis, we utilize the comprehensive collection of labeled benchmark datasets introduced in the ADBench study \cite{han2022adbench}. The anomaly detection methods involved in the comparison are PCA \cite{shyu2003novel}, OCSVM \cite{scholkopf1999support}, KNN \cite{ramaswamy2000efficient}, LOF \cite{breunig2000lof}, CBLOF \cite{he2003discovering}, HBOS \cite{goldstein2012histogram}, COPOD \cite{li2020copod}, ECOD \cite{li2022ecod}, and IF \cite{liu2008iforest}. Each method was executed 10 times using different random seeds and the final
results were averaged. For the RIFT ensemble we used $M=100$ trees,
$\psi=256$ samples per tree and no feature bagging (ratio $1.0$), which are the
default ensemble size and sub-sample size of the IF, so that the
two ensembles are directly comparable. These values were fixed in advance and
not tuned per dataset. The baseline methods were run with the default
hyper-parameters of PyOD \cite{zhao2019pyod}. All methods received the same
preprocessing (median/MAD scaling, described in Section \ref{sec:numerical}).

Certain benchmark datasets are exceptionally large in terms of either sample size or dimensionality, preventing specific algorithms from completing execution within a reasonable timeframe. Consequently, for datasets exceeding 20,000 instances or 500 features, PCA, OCSVM, KNN, LOF, and the exact version of our proposed RIFT algorithm were omitted from the evaluation.

The detailed numerical results of the AUC comparison are summarized in Table \ref{tab:auc_cmp}.

\begin{sidewaystable}
\footnotesize
\begin{tabular}{l|llllllllll|lll}
\toprule
 & PCA & OCSVM & KNN & LOF & CBLOF & HBOS & COPOD & ECOD & IF & LODA & \multicolumn{3}{c}{RIFT} \\
Dataset &  &  &  &  &  &  &  &  &  &  &  Ensemble & Exact & KNN \\
\midrule
ALOI & - & - & - & - & 0.557 & 0.529 & 0.515 & 0.531 & 0.539 & 0.497 & 0.558 & - & 0.558 \\
annthyroid & 0.677 & 0.854 & 0.909 & 0.523 & 0.903 & 0.620 & 0.776 & 0.789 & 0.819 & 0.533 & 0.888 & 0.887 & 0.887 \\
backdoor & - & - & - & - & 0.789 & 0.739 & 0.789 & 0.846 & 0.740 & 0.682 & 0.848 & - & 0.833 \\
breastw & 0.946 & 0.917 & 0.977 & 0.385 & 0.965 & 0.985 & 0.994 & 0.991 & 0.987 & 0.988 & 0.971 & 0.962 & 0.962 \\
campaign & - & - & - & - & 0.771 & 0.773 & 0.783 & 0.770 & 0.713 & 0.564 & 0.796 & - & 0.801 \\
cardio & 0.951 & 0.887 & 0.692 & 0.562 & 0.814 & 0.841 & 0.922 & 0.935 & 0.931 & 0.847 & 0.892 & 0.895 & 0.895 \\
Cardiotocography & 0.752 & 0.756 & 0.512 & 0.552 & 0.572 & 0.599 & 0.663 & 0.785 & 0.682 & 0.708 & 0.758 & 0.766 & 0.766 \\
celeba & - & - & - & - & 0.645 & 0.755 & 0.751 & 0.757 & 0.694 & 0.548 & 0.725 & - & 0.671 \\
census & - & - & - & - & 0.695 & 0.620 & 0.674 & 0.660 & 0.608 & 0.485 & 0.648 & - & - \\
cover & - & - & - & - & 0.897 & 0.762 & 0.884 & 0.920 & 0.880 & 0.876 & 0.862 & - & 0.835 \\
donors & - & - & - & - & 0.682 & 0.736 & 0.815 & 0.888 & 0.776 & 0.533 & 0.724 & - & 0.779 \\
fault & 0.484 & 0.452 & 0.649 & 0.591 & 0.562 & 0.599 & 0.455 & 0.469 & 0.575 & 0.440 & 0.434 & 0.433 & 0.433 \\
fraud & - & - & - & - & 0.953 & 0.954 & 0.947 & 0.950 & 0.950 & 0.670 & 0.955 & - & 0.955 \\
glass & 0.709 & 0.866 & 0.887 & 0.809 & 0.868 & 0.780 & 0.755 & 0.705 & 0.786 & 0.575 & 0.831 & 0.831 & 0.831 \\
Hepatitis & 0.755 & 0.683 & 0.722 & 0.750 & 0.657 & 0.769 & 0.804 & 0.739 & 0.709 & 0.577 & 0.735 & 0.735 & 0.735 \\
http & - & - & - & - & 0.999 & 0.994 & 0.992 & 0.979 & 1.000 & 0.587 & 0.999 & - & 0.054 \\
InternetAds & - & - & - & - & 0.703 & 0.684 & 0.676 & 0.677 & 0.688 & 0.566 & 0.690 & - & 0.692 \\
Ionosphere & 0.788 & 0.911 & 0.929 & 0.893 & 0.892 & 0.542 & 0.789 & 0.728 & 0.848 & 0.817 & 0.847 & 0.846 & 0.846 \\
landsat & 0.365 & 0.359 & 0.578 & 0.550 & 0.576 & 0.555 & 0.421 & 0.368 & 0.484 & 0.385 & 0.443 & 0.461 & 0.461 \\
letter & 0.526 & 0.506 & 0.901 & 0.896 & 0.769 & 0.596 & 0.560 & 0.572 & 0.636 & 0.540 & 0.579 & 0.582 & 0.582 \\
Lymphography & 0.996 & 0.980 & 0.980 & 0.979 & 0.976 & 0.993 & 0.996 & 0.995 & 0.999 & 0.786 & 0.982 & 0.982 & 0.982 \\
magic.gamma & 0.673 & 0.682 & 0.792 & 0.696 & 0.700 & 0.716 & 0.681 & 0.638 & 0.723 & 0.652 & 0.695 & 0.693 & 0.693 \\
mammography & 0.888 & 0.888 & 0.873 & 0.722 & 0.858 & 0.801 & 0.905 & 0.906 & 0.858 & 0.873 & 0.859 & 0.836 & 0.836 \\
mnist & 0.851 & 0.815 & 0.796 & 0.646 & 0.793 & 0.618 & 0.774 & 0.746 & 0.797 & 0.730 & 0.840 & 0.845 & 0.845 \\
musk & 1.000 & 1.000 & 0.756 & 0.560 & 1.000 & 1.000 & 0.946 & 0.956 & 1.000 & 0.988 & 1.000 & 1.000 & 1.000 \\
optdigits & 0.521 & 0.563 & 0.419 & 0.586 & 0.804 & 0.846 & 0.682 & 0.605 & 0.722 & 0.499 & 0.600 & 0.600 & 0.600 \\
PageBlocks & 0.904 & 0.910 & 0.911 & 0.722 & 0.915 & 0.768 & 0.875 & 0.914 & 0.900 & 0.722 & 0.913 & 0.911 & 0.911 \\
pendigits & 0.934 & 0.943 & 0.761 & 0.516 & 0.911 & 0.924 & 0.905 & 0.927 & 0.942 & 0.936 & 0.954 & 0.958 & 0.958 \\
Pima & 0.637 & 0.620 & 0.702 & 0.629 & 0.646 & 0.706 & 0.654 & 0.594 & 0.673 & 0.637 & 0.674 & 0.674 & 0.674 \\
satellite & 0.602 & 0.606 & 0.655 & 0.545 & 0.728 & 0.756 & 0.633 & 0.583 & 0.705 & 0.620 & 0.677 & 0.692 & 0.692 \\
satimage-2 & 0.977 & 0.985 & 0.924 & 0.546 & 0.999 & 0.980 & 0.974 & 0.965 & 0.994 & 0.985 & 0.994 & 0.994 & 0.994 \\
shuttle & - & - & - & - & 0.764 & 0.984 & 0.995 & 0.993 & 0.997 & 0.567 & 0.994 & - & 0.994 \\
skin & - & - & - & - & 0.594 & 0.587 & 0.471 & 0.489 & 0.674 & 0.448 & 0.620 & - & 0.585 \\
smtp & - & - & - & - & 0.867 & 0.806 & 0.912 & 0.880 & 0.904 & 0.822 & 0.900 & - & 0.886 \\
SpamBase & 0.550 & 0.563 & 0.545 & 0.433 & 0.597 & 0.652 & 0.688 & 0.656 & 0.629 & 0.449 & 0.599 & 0.607 & 0.607 \\
speech & 0.469 & 0.466 & 0.484 & 0.508 & 0.468 & 0.476 & 0.491 & 0.470 & 0.473 & 0.462 & 0.470 & 0.469 & 0.469 \\
Stamps & 0.906 & 0.864 & 0.851 & 0.466 & 0.761 & 0.893 & 0.930 & 0.876 & 0.890 & 0.899 & 0.888 & 0.888 & 0.888 \\
thyroid & 0.958 & 0.989 & 0.982 & 0.561 & 0.984 & 0.941 & 0.939 & 0.977 & 0.979 & 0.869 & 0.989 & 0.989 & 0.989 \\
vertebral & 0.439 & 0.433 & 0.381 & 0.457 & 0.394 & 0.303 & 0.335 & 0.420 & 0.355 & 0.298 & 0.374 & 0.374 & 0.374 \\
vowels & 0.597 & 0.766 & 0.969 & 0.935 & 0.893 & 0.675 & 0.496 & 0.593 & 0.755 & 0.699 & 0.753 & 0.754 & 0.754 \\
Waveform & 0.647 & 0.595 & 0.740 & 0.711 & 0.728 & 0.690 & 0.734 & 0.603 & 0.720 & 0.615 & 0.724 & 0.725 & 0.725 \\
WBC & 0.993 & 0.995 & 0.992 & 0.805 & 0.980 & 0.989 & 0.995 & 0.995 & 0.995 & 0.990 & 0.995 & 0.995 & 0.995 \\
WDBC & 0.986 & 0.973 & 0.969 & 0.983 & 0.913 & 0.992 & 0.994 & 0.971 & 0.988 & 0.975 & 0.981 & 0.982 & 0.982 \\
Wilt & 0.378 & 0.478 & 0.689 & 0.665 & 0.553 & 0.350 & 0.345 & 0.394 & 0.462 & 0.354 & 0.485 & 0.485 & 0.485 \\
wine & 0.787 & 0.769 & 0.587 & 0.889 & 0.267 & 0.907 & 0.867 & 0.733 & 0.802 & 0.871 & 0.880 & 0.880 & 0.880 \\
WPBC & 0.481 & 0.470 & 0.501 & 0.504 & 0.471 & 0.554 & 0.523 & 0.481 & 0.498 & 0.500 & 0.477 & 0.477 & 0.477 \\
yeast & 0.397 & 0.408 & 0.401 & 0.464 & 0.450 & 0.402 & 0.381 & 0.444 & 0.397 & 0.461 & 0.379 & 0.381 & 0.381 \\
\bottomrule
\end{tabular}
\caption{ROCAUC scores on the ADBench classic benchmark datasets}
\label{tab:auc_cmp}
\end{sidewaystable}

As demonstrated by the ADBench study, it is impossible to declare a universal winner, no single anomaly detection method consistently outperforms the rest. The effectiveness of these algorithms is highly data-dependent, with certain methods excelling on specific datasets while struggling on others. This inherent variability underscores the need for developing algorithms based on novel theoretical paradigms, such as our proposed approach. The evaluation results across all datasets are summarized in Table \ref{tab:auc_all}, while Table \ref{tab:auc_small} presents the performance metrics specifically for the smaller datasets.

\begin{table}
    \centering
\begin{tabular}{lcc}
\toprule
 & Mean AUC & Relative std of AUC \\
\midrule
IF & 0.7561 & 0.0220 \\
RIFT exact & 0.7526 & 0.0000 \\
RIFT KNN & 0.7526 & 0.0000 \\
RIFT ensemble & 0.7518 & 0.0010 \\
KNN & 0.7476 & 0.0000 \\
CBLOF & 0.7461 & 0.0384 \\
OCSVM & 0.7339 & 0.0000 \\
COPOD & 0.7319 & 0.0000 \\
HBOS & 0.7299 & 0.0000 \\
PCA & 0.7213 & 0.0000 \\
ECOD & 0.7213 & 0.0000 \\
LODA & 0.6847 & 0.0890 \\
LOF & 0.6481 & 0.0000 \\
\bottomrule
\end{tabular}
    \caption{Results for datasets with at most 20,000 samples and 500 features}
    \label{tab:auc_small}
\end{table}

On the smaller datasets, the IF (introduced in 2008) and all RIFT variants perform at essentially the same level, with the Isolation Forest holding a marginal lead (mean AUC of 0.7561 versus 0.7526 and 0.7518 for RIFT exact and RIFT ensemble). On the full collection, which also includes the large datasets where only the scalable methods could be run, the RIFT ensemble achieves the highest mean AUC, narrowly ahead of the Isolation Forest (0.7634 versus 0.7632, they are statistically indistinguishable). RIFT KNN performs extremely close to the exact variant only, it is much faster.
Alongside the average AUC scores, the tables also report the relative standard deviation, calculated by dividing the standard deviation of the results obtained across different random seeds by the mean. For deterministic methods, this value is naturally 0, so both RIFT Exact and RIFT KNN have 0 standard deviation. Notably, the relative standard deviation of the RIFT ensemble is an order of magnitude lower than that of the IF on both dataset collections (0.0010 versus 0.0220, and 0.0014 versus 0.0212), even though the ensemble size was identical in both cases.


\begin{table}
    \centering
\begin{tabular}{lcc}
\toprule
 & Mean AUC & Relative std of AUC \\
\midrule
RIFT ensemble & 0.7634 & 0.0014 \\
IF & 0.7632 & 0.0212 \\
CBLOF & 0.7507 & 0.0426 \\
COPOD & 0.7466 & 0.0000 \\
ECOD & 0.7418 & 0.0000 \\
HBOS & 0.7391 & 0.0000 \\
LODA & 0.6623 & 0.1196 \\
\bottomrule
\end{tabular}
    \caption{Results for all classic datasets of ADBench}
    \label{tab:auc_all}
\end{table}

\begin{sidewaystable}
\footnotesize
\begin{tabular}{l|llllllllll|lll}
\toprule
 & PCA & OCSVM & KNN & LOF & CBLOF & HBOS & COPOD & ECOD & IF & LODA & \multicolumn{3}{c}{RIFT} \\
Dataset &  &  &  &  &  &  &  &  &  &  &  Ensemble & Exact & KNN \\
\midrule
ALOI & - & - & - & - & 0.144 & 0.099 & 0.322 & 0.919 & 0.218 & 0.086 & 1.394 & - & 3.796 \\
annthyroid & 0.001 & 0.877 & 0.038 & 0.063 & 0.007 & 0.001 & 1.661 & 0.048 & 0.116 & 0.016 & 0.155 & 0.412 & 0.057 \\
backdoor & - & - & - & - & 0.530 & 0.187 & 4.068 & 2.027 & 0.441 & 0.664 & 7.669 & - & 53.208 \\
breastw & 0.001 & 0.008 & 0.012 & 0.013 & 0.002 & 0.001 & 0.788 & 0.011 & 0.080 & 0.005 & 0.069 & 0.003 & 0.006 \\
campaign & - & - & - & - & 0.071 & 0.040 & 0.259 & 1.104 & 0.314 & 0.142 & 1.775 & - & 4.627 \\
cardio & 0.001 & 0.079 & 0.010 & 0.011 & 0.003 & 0.002 & 0.014 & 0.013 & 0.085 & 0.007 & 0.108 & 0.030 & 0.012 \\
Cardiotocography & 0.002 & 0.101 & 0.012 & 0.013 & 0.004 & 0.002 & 0.014 & 0.014 & 0.087 & 0.008 & 0.112 & 0.039 & 0.016 \\
celeba & - & - & - & - & 0.277 & 0.138 & 0.496 & 0.487 & 1.320 & 0.630 & 6.494 & - & 44.168 \\
census & - & - & - & - & 5.477 & 1.596 & 16.994 & 16.845 & 1.366 & 4.738 & 50.150 & - & - \\
cover & - & - & - & - & 0.149 & 0.044 & 0.741 & 0.231 & 0.914 & 0.469 & 4.333 & - & 89.256 \\
donors & - & - & - & - & 0.227 & 0.127 & 0.643 & 0.645 & 3.229 & 1.348 & 9.205 & - & 2.412 \\
fault & 0.002 & 0.106 & 0.015 & 0.015 & 0.003 & 0.002 & 0.016 & 0.015 & 0.091 & 0.008 & 0.128 & 0.037 & 0.015 \\
fraud & - & - & - & - & 0.326 & 0.130 & 1.547 & 0.731 & 1.293 & 0.613 & 6.739 & - & 386.518 \\
glass & 0.001 & 0.001 & 0.012 & 0.012 & 0.001 & 0.001 & 1.288 & 0.011 & 0.074 & 0.004 & 0.046 & 0.001 & 0.005 \\
Hepatitis & 0.001 & 0.001 & 0.002 & 0.001 & 0.001 & 0.001 & 0.015 & 0.012 & 0.087 & 0.005 & 0.019 & 0.000 & 0.005 \\
http & - & - & - & - & 0.166 & 0.031 & 1.678 & 0.215 & 2.237 & 0.815 & 5.453 & - & 49.254 \\
InternetAds & - & - & - & - & 0.102 & 0.085 & 2.334 & 0.330 & 0.087 & 0.079 & 1.990 & - & 0.824 \\
Ionosphere & 0.001 & 0.005 & 0.002 & 0.002 & 0.002 & 0.002 & 0.012 & 0.012 & 0.080 & 0.005 & 0.087 & 0.001 & 0.005 \\
landsat & 0.006 & 1.352 & 0.106 & 0.104 & 0.011 & 0.006 & 0.023 & 0.023 & 0.112 & 0.020 & 0.282 & 0.461 & 0.085 \\
letter & 0.002 & 0.079 & 0.012 & 0.013 & 0.004 & 0.003 & 0.014 & 0.014 & 0.088 & 0.008 & 0.117 & 0.025 & 0.013 \\
Lymphography & 0.001 & 0.001 & 0.001 & 0.001 & 0.001 & 0.001 & 0.012 & 0.011 & 0.076 & 0.004 & 0.032 & 0.000 & 0.005 \\
magic.gamma & 0.003 & 7.633 & 0.171 & 0.268 & 0.012 & 0.004 & 0.532 & 0.022 & 0.186 & 0.040 & 0.376 & 3.322 & 0.364 \\
mammography & 0.002 & 2.183 & 0.091 & 0.102 & 0.008 & 0.002 & 1.588 & 0.053 & 0.134 & 0.023 & 0.203 & 0.981 & 0.067 \\
mnist & 0.045 & 4.835 & 0.414 & 0.393 & 0.046 & 0.014 & 2.127 & 0.065 & 0.130 & 0.037 & 0.518 & 1.076 & 0.313 \\
musk & 0.052 & 1.179 & 0.134 & 0.128 & 0.015 & 0.016 & 0.037 & 0.037 & 0.098 & 0.022 & 0.412 & 0.245 & 0.122 \\
optdigits & 0.013 & 1.441 & 0.151 & 0.144 & 0.015 & 0.009 & 0.031 & 0.031 & 0.113 & 0.023 & 0.291 & 0.380 & 0.111 \\
PageBlocks & 0.001 & 0.521 & 0.025 & 0.035 & 0.003 & 0.002 & 0.521 & 0.015 & 0.104 & 0.011 & 0.147 & 0.255 & 0.038 \\
pendigits & 0.002 & 1.012 & 0.073 & 0.054 & 0.006 & 0.003 & 0.019 & 0.017 & 0.125 & 0.021 & 0.191 & 0.452 & 0.063 \\
Pima & 0.001 & 0.010 & 0.013 & 0.013 & 0.002 & 0.001 & 1.858 & 0.012 & 0.081 & 0.005 & 0.069 & 0.004 & 0.007 \\
satellite & 0.006 & 1.380 & 0.120 & 0.105 & 0.012 & 0.006 & 0.026 & 0.024 & 0.111 & 0.021 & 0.282 & 0.460 & 0.086 \\
satimage-2 & 0.006 & 1.117 & 0.091 & 0.087 & 0.010 & 0.005 & 0.881 & 0.183 & 0.110 & 0.019 & 0.260 & 0.373 & 0.073 \\
shuttle & - & - & - & - & 0.019 & 0.008 & 0.807 & 0.046 & 0.321 & 0.072 & 0.796 & - & 2.356 \\
skin & - & - & - & - & 0.065 & 0.013 & 1.548 & 0.066 & 0.675 & 0.316 & 2.377 & - & 2.212 \\
smtp & - & - & - & - & 0.030 & 0.006 & 0.030 & 0.030 & 0.508 & 0.149 & 0.970 & - & 4.601 \\
SpamBase & 0.009 & 0.852 & 0.083 & 0.081 & 0.008 & 0.005 & 2.014 & 0.034 & 0.093 & 0.015 & 0.252 & 0.243 & 0.071 \\
speech & 0.880 & 4.656 & 0.634 & 0.610 & 0.086 & 0.042 & 0.107 & 0.106 & 0.102 & 0.047 & 0.825 & 0.676 & 0.248 \\
Stamps & 0.001 & 0.002 & 0.012 & 0.012 & 0.001 & 0.001 & 0.777 & 0.011 & 0.081 & 0.004 & 0.067 & 0.001 & 0.005 \\
thyroid & 0.001 & 0.227 & 0.023 & 0.027 & 0.003 & 0.001 & 1.560 & 0.049 & 0.098 & 0.010 & 0.108 & 0.115 & 0.023 \\
vertebral & 0.001 & 0.002 & 0.012 & 0.012 & 0.002 & 0.001 & 0.011 & 0.011 & 0.075 & 0.004 & 0.053 & 0.001 & 0.005 \\
vowels & 0.001 & 0.039 & 0.013 & 0.013 & 0.002 & 0.001 & 2.065 & 0.015 & 0.086 & 0.008 & 0.086 & 0.015 & 0.010 \\
Waveform & 0.002 & 0.272 & 0.025 & 0.029 & 0.006 & 0.002 & 0.015 & 0.015 & 0.098 & 0.013 & 0.140 & 0.112 & 0.031 \\
WBC & 0.001 & 0.001 & 0.012 & 0.012 & 0.001 & 0.001 & 0.770 & 0.011 & 0.076 & 0.004 & 0.047 & 0.001 & 0.005 \\
WDBC & 0.001 & 0.005 & 0.002 & 0.002 & 0.002 & 0.002 & 0.014 & 0.012 & 0.083 & 0.005 & 0.085 & 0.002 & 0.006 \\
Wilt & 0.001 & 0.352 & 0.016 & 0.025 & 0.004 & 0.001 & 1.537 & 0.058 & 0.103 & 0.011 & 0.113 & 0.182 & 0.030 \\
wine & 0.001 & 0.001 & 0.013 & 0.012 & 0.001 & 0.001 & 2.038 & 0.015 & 0.078 & 0.004 & 0.026 & 0.000 & 0.005 \\
WPBC & 0.001 & 0.002 & 0.001 & 0.001 & 0.001 & 0.002 & 0.012 & 0.011 & 0.078 & 0.004 & 0.052 & 0.001 & 0.005 \\
yeast & 0.001 & 0.035 & 0.013 & 0.013 & 0.003 & 0.001 & 1.032 & 0.012 & 0.083 & 0.007 & 0.077 & 0.015 & 0.010 \\
\bottomrule
\end{tabular}
\caption{Runtime (seconds) on the ADBench classic benchmark datasets}
\label{tab:runtime_all}
\end{sidewaystable}

Table \ref{tab:runtime_all} reports the execution times of all methods. As with the IF, the runtime of the RIFT ensemble depends strongly on the ensemble size and on the number of samples used per tree. We deliberately set both parameters to the defaults of the Isolation Forest, so that the two ensembles are directly comparable. Even so, the RIFT ensemble is slower than the IF, because building a single RIFT tree requires more computation than building an isolation tree. In exchange, each RIFT tree depends on fewer random choices, which explains why RIFT has a lower variance but a higher runtime. The slowdown is most pronounced on high-dimensional data, where computing Euclidean distances dominates the cost: on InternetAds, census and backdoor, the RIFT ensemble is 10--30 times slower than the IF, whereas on low-dimensional datasets such as http, smtp and skin the gap is only a factor of 2--4. The execution time can, however, be reduced substantially. As shown in Table \ref{tab:ens}, the number of trees can be decreased drastically with only a minimal loss in AUC, while the standard deviation stays low, reducing the ensemble size from 256 to 4 trees shortens the execution time by nearly two orders of magnitude. Sub-sampling (Table \ref{tab:smpl}) and feature bagging (Table \ref{tab:feat}) provide further possibilities for reducing the computational cost.

\section{Conclusions}\label{sec:conclusions}

In this paper we introduced an unsupervised, deep-learning-free anomaly detection method derived from the geometric interpretation of the recently established closed-form solution of the one-dimensional IF. By representing the data through its Euclidean Minimum Spanning Tree and scoring each point as the sum of the apparent sizes of the tree edges observed from it, this method generalizes the 1D result to arbitrary dimensions: in one dimension it recovers the exact expected output of an IF with infinitely many trees, while in higher dimensions it replaces the random recursive partitioning of IF with a deterministic, distance-based construction. The resulting score is inherently density-invariant and free of the axis-parallel artifacts characteristic of IF.

We presented three variants of the method: an exact formulation built on a
global MST, a scalable approximation based on the $k$-nearest-neighbor graph,
and a sub-sampled ensemble with feature bagging. The exact and approximate
variants are fully deterministic, while the ensemble trades determinism for
scalability and robustness, yet confines its randomness to the instance- and
feature-sampling stage rather than to every split, as in IF.

On synthetic data and the ADBench benchmark suite, RIFT achieves accuracy
competitive with IF. Consistent with the central observation of
ADBench -- that no single method dominates across all datasets -- RIFT is not a
universal winner, but it matches the strongest classical detectors while
offering two distinct advantages: a fully deterministic variant, and an
ensemble whose score variance across random seeds is an order of magnitude
lower than that of IF at the same ensemble size.

\section*{Acknowledgement}

The project supported by the Doctoral Excellence Fellowship Programme (DCEP) is funded by the National Research Development and Innovation Fund of the Ministry of Culture and Innovation and the Budapest University of Technology and Economics, under a grant agreement with the National Research, Development and Innovation Office.

\bibliography{bib}

\begin{appendices}

\section{The derivation of the explicit solution of the 1D Isolation Forest}\label{sec:proof1d}

In this section, we derive the expected path length of a training sample $x_i$ in an isolation tree. We assume that 1) the sample values are pairwise distinct and ordered as $x_1<x_2<\dots<x_N$, 2) each tree is created on all samples ($\psi=N$) without a height limit so that the recursion stops only when a node contains a single point, 3) at every node the split value is drawn uniformly from $[\min_{x\in S}x,\max_{x\in S}x]$ as in the original algorithm. We introduce $H(i) = E[\mathcal{H}(x_i\vert\mathcal{X})]$ as the limit of the forest-averaged path length as the number of trees tends to infinity. The same result is obtained in \cite{zheng2025iforestbias}, our aim here is to show that it can also be derived in a few elementary algebraic steps.

We first express the expected path length of $x_i$ given a continuous subset of ordered data consisting of only elements $x_\ell,\dots,x_u$ with $\ell\le i\le u$, that is, $H_{\ell,u}(i) = E[\mathcal{H}(x_i\vert x_\ell,x_{\ell+1},\dots,x_u)]$. From this quantity, we immediately get $H(i)$, since $H(i)=H_{1,N}(i)$. By conditioning on the interval of the first random split, the expected path length of $x_i$ can be expressed recursively as
\begin{align}\label{eq:hlui1}
    H_{\ell,u}(i) = 1+\sum_{k=\ell}^{i-1}\frac{x_{k+1}-x_k}{x_u-x_\ell}H_{k+1,u}(i) + \sum_{k=i}^{u-1}\frac{x_{k+1}-x_k}{x_u-x_\ell}H_{\ell,k}(i),
\end{align}
for $\ell<u$, with the boundary condition $H_{i,i}(i) = 0$. Here, the first term $1$ is due to the split itself (increasing the path length by $1$) and $(x_{k+1}-x_k)/(x_u-x_\ell)$ is the probability that the first split falls into the $(x_{k+1},x_k)$ interval. If this interval is before $x_i$, the problem can be reduced to finding the place of $x_i$ in the tree consisting of elements $x_{k+1}, \dots, x_u$ (represented by the second term of \eqref{eq:hlui1}), if it is past $x_i$, the problem can be reduced to the first part of the sequence $x_{\ell}, \dots, x_k$ (third term of \eqref{eq:hlui1}, see Figure \ref{fig:probsplit}).

\begin{figure}
    \centering
    \includegraphics[width=\textwidth]{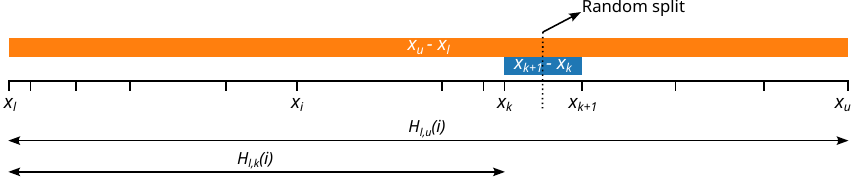}
    \caption{Splitting between $x_k$ and $x_{k+1}$}
    \label{fig:probsplit}
\end{figure}

The solution of the recursion is formulated by the following theorem.
\begin{theorem}\label{thm:hi}
The expected path length of $x_i$, $H(i)$, is given by the following explicit formula:
\begin{align}\label{eq:hi}
    H(i) = \sum_{k=1}^{i-1}\frac{x_{k+1}-x_k}{x_i-x_k} + \sum_{k=i}^{N-1}\frac{x_{k+1}-x_k}{x_{k+1}-x_i}.
\end{align}
\end{theorem}

\begin{proof}
First we show that for $H_{\ell,u}(i)$, defined by eq. \eqref{eq:hlui1},
\begin{align}\label{eq:hlui}
    H_{\ell,u}(i) = \sum_{k=\ell}^{i-1}\frac{x_{k+1}-x_k}{x_i-x_k} + \sum_{k=i}^{u-1}\frac{x_{k+1}-x_k}{x_{k+1}-x_i}
\end{align}
holds, which, since $H(i)=H_{1,N}(i)$, proves the Theorem. 

To show that \eqref{eq:hlui} satisfies \eqref{eq:hlui1}, we substitute \eqref{eq:hlui} into the right-hand side of \eqref{eq:hlui1}, leading to
\begin{align}
    \begin{split}
    H_{\ell,u}(i) &= 1\\
    &\quad+\sum_{k=\ell}^{i-1}\frac{x_{k+1}-x_k}{x_u-x_\ell} \sum_{j=k+1}^{i-1}\frac{x_{j+1}-x_j}{x_i-x_j}\\
    &\quad+\sum_{k=\ell}^{i-1}\frac{x_{k+1}-x_k}{x_u-x_\ell} \sum_{j=i}^{u-1}\frac{x_{j+1}-x_j}{x_{j+1}-x_i}\\
    &\quad+\sum_{k=i}^{u-1}\frac{x_{k+1}-x_k}{x_u-x_\ell} \sum_{j=\ell}^{i-1}\frac{x_{j+1}-x_j}{x_i-x_j}\\
    &\quad+\sum_{k=i}^{u-1}\frac{x_{k+1}-x_k}{x_u-x_\ell} \sum_{j=i}^{k-1}\frac{x_{j+1}-x_j}{x_{j+1}-x_i}.\\
    \end{split}
\end{align}
Swapping the two sums in all four terms, we get
\begin{align}
    \begin{split}
    H_{\ell,u}(i) &= 1\\
    &\quad+\sum_{j=\ell}^{i-1}\frac{x_{j+1}-x_j}{x_i-x_j} \sum_{k=\ell}^{j-1}\frac{x_{k+1}-x_k}{x_u-x_\ell} \\
    &\quad+\sum_{j=i}^{u-1}\frac{x_{j+1}-x_j}{x_{j+1}-x_i} \sum_{k=\ell}^{i-1}\frac{x_{k+1}-x_k}{x_u-x_\ell}   \\
    &\quad+\sum_{j=\ell}^{i-1}\frac{x_{j+1}-x_j}{x_i-x_j} \sum_{k=i}^{u-1}\frac{x_{k+1}-x_k}{x_u-x_\ell}  \\
    &\quad+\sum_{j=i}^{u-1}\frac{x_{j+1}-x_j}{x_{j+1}-x_i} \sum_{k=j+1}^{u-1}\frac{x_{k+1}-x_k}{x_u-x_\ell}.\\
    \end{split}
\end{align}
In all four terms, the inner sums collapse as $\sum_{k=\ell}^{j-1}\frac{x_{k+1}-x_k}{x_u-x_\ell}=\frac{x_j-x_\ell}{x_u-x_\ell}$, $\sum_{k=\ell}^{i-1}\frac{x_{k+1}-x_k}{x_u-x_\ell}=\frac{x_i-x_\ell}{x_u-x_\ell}$, 
$\sum_{k=i}^{u-1}\frac{x_{k+1}-x_k}{x_u-x_\ell}=\frac{x_u-x_i}{x_u-x_\ell}$ and $\sum_{k=j+1}^{u-1}\frac{x_{k+1}-x_k}{x_u-x_\ell}=\frac{x_u-x_{j+1}}{x_u-x_\ell}$, yielding
\begin{align}
    \begin{split}
    H_{\ell,u}(i) &= 1\\
    &\quad+\sum_{j=\ell}^{i-1}\frac{x_{j+1}-x_j}{x_i-x_j} \left(\frac{x_j-x_\ell}{x_u-x_\ell} + \frac{x_u-x_i}{x_u-x_\ell}\right) \\
    &\quad+\sum_{j=i}^{u-1}\frac{x_{j+1}-x_j}{x_{j+1}-x_i} \left(\frac{x_i-x_\ell}{x_u-x_\ell} + \frac{x_u-x_{j+1}}{x_u-x_\ell} \right),  \\
    \end{split}
\end{align}
which, exploiting that the terms in parentheses are equal to $1-\frac{x_i-x_j}{x_u-x_\ell}$ and $1-\frac{x_{j+1}-x_{i}}{x_u-x_\ell}$, gives
\begin{align}
    H_{\ell,u}(i) &= 1 + \sum_{j=\ell}^{i-1}\frac{x_{j+1}-x_j}{x_i-x_j} + \sum_{j=i}^{u-1}\frac{x_{j+1}-x_j}{x_{j+1}-x_i} \underbrace{- \frac{x_i-x_\ell}{x_u-x_\ell}-\frac{x_u-x_i}{x_u-x_\ell}}_{-1},
\end{align}
resulting \eqref{eq:hlui}.

\end{proof}

\end{appendices}

\end{document}